\documentclass[11pt]{article}
\usepackage{amsmath,amssymb}
\usepackage{latexsym}
\usepackage{graphicx} 
\usepackage{algorithm}
\usepackage{amsthm}
\usepackage{authblk}

\newtheorem{definition}{Definition}
\newtheorem{theorem}[definition]{Theorem}
\newtheorem{corollary}[definition]{Corollary}
\newtheorem{lemma}[definition]{Lemma}
\newtheorem{proposition}[definition]{Proposition}

\newtheorem{remark}[definition]{Remark}
\newtheorem{example}[definition]{Example}

\newcommand{\lnon}{\overline}
\newcommand{\lxor}{\oplus}
\newcommand{\avec}{{\bf a}}

\newcommand{\evec}{{\bf e}}
\newcommand{\xvec}{{\bf x}}

\newcommand{\yvec}{{\bf y}}
\newcommand{\zvec}{{\bf z}}
\newcommand{\fvec}{{\bf f}}
\newcommand{\gvec}{{\bf g}}
\newcommand{\wvec}{{\bf w}}

\begin{document}

\title{On the Compressive Power of Boolean Threshold Autoencoders}
\author[1]{Avraham A. Melkman}
\author[2]{Sini Guo}
\author[2]{Wai-Ki Ching\thanks{
partially supported by Hong Kong RGC GRF Grant no. 17301519, IMR and RAE Research fund from Faculty of Science, HKU.}}
\author[3]{Pengyu Liu}
\author[3]{Tatsuya Akutsu\thanks{partially supported by Grant-in-Aid \#18H04113 from JSPS, Japan,}\thanks{Correspoding author. e-mail: takutsu@kuicr.kyoto-u.ac.jp}}
\affil[1]{Department of Computer Science,
Ben-Gurion University of the Negev}
\affil[2]{Department of Mathematics,
The University of Hong Kong}
\affil[3]{Bioinformatics Center, Institute for Chemical Research, Kyoto University}

\maketitle

\begin{abstract}%
An autoencoder is a layered neural network
whose structure can be viewed as consisting of an \emph{encoder}, which compresses an input vector of dimension $D$ to
a vector of low dimension $d$, and a \emph{decoder} which transforms
the low-dimensional vector back to the original input vector (or one that is very similar). In this paper we explore the compressive power of autoencoders
that are Boolean threshold networks by
studying
the numbers of nodes and layers that are required to ensure that
 each vector in a given set of distinct input binary vectors is transformed back to its original. We show that for any set of $n$ distinct vectors
there exists a seven-layer autoencoder with the smallest possible middle layer,
(i.e., its size is logarithmic in $n$),
but that there is a set of $n$ vectors for which there is no three-layer
autoencoder with a middle layer of the same size.
In addition we present a kind of trade-off: if a considerably larger middle layer is permissible then a five-layer autoencoder does exist.
We also study encoding by itself.
The results we obtain suggest that it is the decoding
that constitutes the bottleneck of autoencoding.
For example, there always is a three-layer Boolean threshold encoder that compresses
$n$ vectors into
a dimension that is reduced to twice the logarithm of $n$.

\end{abstract}

\section{Introduction}
Artificial neural networks have been extensively studied in recent years.
Among various models, \emph{autoencoders} attract much attention because
of
their power to generate new objects such as image data.
An autoencoder is a layered neural network
whose structure can be viewed as consisting of two parts, an \emph{encoder} and
a \emph{decoder}, where the former transforms an input vector to
a low-dimensional vector and the latter transforms the low-dimensional
vector to an output vector which should be
the same as or similar to the input vector, \cite{ackley85,baldi89,hinton06,baldi12}.
An autoencoder is trained in an unsupervised manner to minimize
the difference between input and output data by adjusting weights
(and some other parameters).
In the process it learns, therefore, a mapping from high-dimensional
input data to a low-dimensional representation space.
Although autoencoders have a long history \cite{ackley85,hinton06},
recent studies focus on \emph{variational autoencoders}
\cite{kingma13,doersch16,tschannen18}
because of their generative power.
Autoencoders have been applied to various areas including
image processing \cite{doersch16,tschannen18},
natural language processing \cite{tschannen18},
and drug discovery \cite{gomez18}.

As described above, autoencoders perform dimensionality reduction,
a kind of data compression.
However, how data are compressed via autoencoders is not yet very clear.
Of course, extensive studies have been done on the representation
power of deep neural networks \cite{delalleau11,montufar14,an17,zhang17}.
Yet, to the best of
the authors' knowledge, the quantitative relationship
between the compressive power and the numbers of layers and nodes in
autoencoders is still unclear.

In this paper, we study the compressive power of autoencoders
using a Boolean model of layered neural networks.
In this model, each node takes on values that are
either 1 (active) or 0 (inactive)
and the activation rule for each node is given by a Boolean function.
That is, we consider a \emph{Boolean network} (BN) \cite{kauffman69}
as a model of a neural network.
BNs have been used as a discrete model of genetic networks,
for which extensive studies have been done on inference, control,
and analysis \cite{akutsu18,cheng11,li16,liu16,lu17,zhao16}.
It should be noted that BNs have also been used as a discrete model
of neural networks in which functions are restricted to be
linear Boolean threshold functions whose outputs are determined by
comparison of the weighted sum of input values with a threshold
\cite{anthony01}.
Such a BN is referred to as \emph{Boolean threshold networks} (BTNs)
in this paper.
Although extensive theoretical studies have been
devoted to the representational power of BTNs \cite{anthony01},
almost none considered BN models of autoencoders with the notable exception of
Baldi's study of the computational complexity of clustering via
autoencoders \cite{baldi12}.

\begin{table*}
\begin{small}
\begin{center}
\caption{Summary of Results.}
\label{tbl:summary}
\begin{tabular}{|l|llll|}
\hline
& $d$ & architecture & type & constraint \\
\hline
Proposition 4 & $\log n$ & $D/d/D$ & Encoder/Decoder & BN
\\
\hline
Theorem 5 & $\lceil 8 \sqrt{2M} \ln n \rceil$ & $D/d$ & Encoder & BTN, $M = \#1\mbox{s in each }\xvec^i$\\
Theorem 8 & $2 \lceil \log n \rceil$ & $D/d$ & Encoder & BN with parity function \\
Theorem 9 & $2 \lceil \log n \rceil$ & $D/D^2/d$ & Encoder & BTN \\
Theorem 15 & $ \lceil \log n \rceil$ & 4 layers ($O(\sqrt{n}+D)$ nodes) & Encoder & BTN \\
\hline
Theorem 19 & $2 \lceil \sqrt{n} \rceil$ & $D/({\frac d 2}+D)/d/{\frac {dD} {2}}/D$
& Encoder/Decoder & BTN \\
Theorem 21 & $\lceil \log n \rceil$ &  $D/n/d/n/D$ & Encoder/Decoder & BTN \\
Theorem 22 & $2\lceil \log \sqrt{n} \rceil$ & 7 layers ($O(D \sqrt{n})$ nodes) & Encoder/Decoder & BTN \\
\hline
\end{tabular}
\end{center}
\end{small}
\end{table*}

The BTNs we consider in this paper are always layered ones,
and we study their compressive power in two settings.
The first one focuses on encoding.
We are given a set of $n$ $D$-dimensional different binary vectors
$X_n=\{\xvec^0,\ldots,\xvec^{n-1}\}$ and the task is to find a BTN
which maps each $\xvec^i$ to some $d$-dimensional binary vector
$\fvec(\xvec^i)$ in such a way that
$\fvec(\xvec^i) \neq \fvec(\xvec^j)$ for all $i \neq j$.
Such a BTN is called a \emph{perfect encoder} for $X_n$.
The second one involves both encoding and decoding.
In this case
we are again given a set of $D$-dimensional binary vectors
$X_n=\{\xvec^0,\ldots,\xvec^{n-1}\}$,
and the task is to find a BTN consisting of an encoder function $\fvec$
and a decoder function $\gvec$ satisfying $\gvec(\fvec(\xvec^i))= \xvec^i$
for all $i=0,\ldots,n-1$.
Such a BTN is called a \emph{perfect autoencoder} for $X_n$.
It is clear from the definition that if $(\fvec,\gvec)$
is a perfect autoencoder,
then $\fvec$ is a perfect encoder.
In this paper, we are interested in whether or not there exists a BTN
for any $X_n$ under the condition that
the architecture (i.e., the number of layers and
the number of nodes in each layer) is fixed.
This setting is reasonable because learning of an autoencoder
is usually performed, for a given set of input vectors,
by fixing the architecture of the BTN and adjusting
the weights of the edges.
Therefore, the existence of a perfect BTN for any $X_n$ implies that
a BTN can be trained for any $X_n$ so that the required conditions are
satisfied if an appropriate set of initial parameters is given.

The results on existence of perfect
encoders and autoencoders
are summarized in Table~\ref{tbl:summary},
in which an entry in the architecture column lists the number of nodes in each layer
(from input to output).
Note that $\log n$ means $\log_2 n$ throughout the paper.
In addition to these positive results,
we show a negative result
(Theorem 16)
stating that for some $X_n$ with $d = \lceil \log n \rceil$,
there does not exist a perfect autoencoder consisting of three layers.

\section{Problem Definitions}

Our first BN model is a three-layered one, see Fig.~\ref{fig:arch3};
its definitions can be extended to networks with four or more layers
in a straightforward way.
Let $\xvec$, $\zvec$, $\yvec$ be binary vectors given to the first,
second, and third layers, respectively.
In this case, the first, second, and third layers correspond to
the input, middle, and output layers, respectively.
Let $\fvec$ and $\gvec$ be lists of Boolean functions assigned to
nodes in the second and third layers, respectively,
which means $\zvec = \fvec(\xvec)$ and $\yvec = \gvec(\zvec)$.
Since our primary interest is in autoencoders,
we assume that $\xvec$ and $\yvec$ are $D$-dimensional binary vectors
and $\zvec$ is a $d$-dimensional binary vector with $d \leq D$.
A list of functions is also referred to as a \emph{mapping}.
The $i$th element of a vector $\xvec$ will be denoted
by $x_i$, which is also used to denote the node corresponding to this element.
Similarly, for each mapping $\fvec$,
$f_i$ denotes the $i$th function.

In the following,
$X_n = \{\xvec^0,\ldots,\xvec^{n-1}\}$
will always denote a set of $n$ $D$-dimensional binary input vectors
that are all different.

\begin{definition}
A mapping $\fvec$: $\{0,1\}^D \rightarrow \{0,1\}^d$
is called a \emph{perfect encoder} for $X_n$
if $\fvec(\xvec^i) \neq$ $\fvec(\xvec^j)$ holds for all $i \neq j$.
\end{definition}

\begin{definition}
A pair of mappings $(\fvec,\gvec)$ with
$\fvec$: $\{0,1\}^D \rightarrow \{0,1\}^d$
and
$\gvec$: $\{0,1\}^d \rightarrow \{0,1\}^D$
is called a \emph{perfect autoencoder} if
$\gvec(\fvec(\xvec^i))=\xvec^i$ holds for all $\xvec^i \in X_n$.
\end{definition}

\begin{figure}[th]
\begin{center}
\includegraphics[width=8cm]{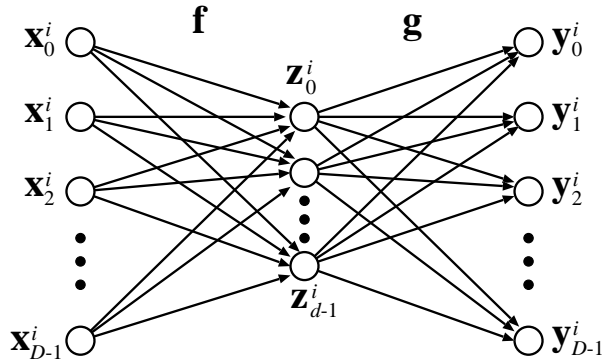}
\caption{Architecture of an autoencoder.}
\label{fig:arch3}
\end{center}
\end{figure}

Fig.~\ref{fig:arch3} illustrates the architecture of an autoencoder.
As mentioned before,
if $(\fvec,\gvec)$ is a perfect autoencoder,
then $\fvec$ is a perfect encoder.

\begin{example}
Let $X_4 = \{\xvec^0,\xvec^1,\xvec^2,\xvec^3\}$ where
$\xvec^0=(0,0,0)$,
$\xvec^1=(1,0,0)$,
$\xvec^2=(1,0,1)$,
$\xvec^3=(1,1,1)$.
Let $D=3$ and $d=2$.
Define $\zvec=\fvec(\xvec)$ and $\yvec=\gvec(\zvec)$ by
	\begin{eqnarray*}
		f_0:x_0 \lxor x_1,~
		f_1:x_2,~
		g_0:z_0 \lor z_1,~
		g_1:\lnon{z_0} \land z_1,~
		g_2:z_1.
	\end{eqnarray*}
This pair of mappings has the following truth table, which shows it to be
a perfect Boolean autoencoder.

\smallskip

\begin{center}
	\begin{tabular}{|lll|ll|lll|}
			\hline
			$x_0$ & $x_1$ & $x_2$ & $z_0$ & $z_1$ & $y_0$ & $y_1$ & $y_2$\\
			\hline
			0 & 0 & 0 & 0 & 0 & 0 & 0 & 0\\
			1 & 0 & 0 & 1 & 0 & 1 & 0 & 0\\
			1 & 0 & 1 & 1 & 1 & 1 & 0 & 1\\
			1 & 1 & 1 & 0 & 1 & 1 & 1 & 1\\
			\hline
	\end{tabular}
\end{center}
\end{example}

\medskip

\begin{proposition}\label{p:map}
For any $X_n$,
there exists a perfect Boolean autoencoder
with $d = \lceil \log n \rceil$.
\label{prop:bool}
\end{proposition}
\begin{proof}
The encoder maps a vector to its index, in binary representation, and the decoder maps the index back to the vector.
To implement the idea formally denote by $i2b_d(j)$ the $d$-dimensional binary vector representing $j$,
for $j < 2^d$ and $ 0<d$.
For example, $i2b_3(5)=(1,0,1)$.

Given $d = \lceil \log n \rceil$ define $(\fvec,\gvec)$
by $\fvec(\xvec^i)=i2b_d(i)$ and $g_j(\fvec(\xvec^i)) = x^i_j$.
Clearly, $(\fvec,\gvec)$ is a perfect autoencoder and
can be represented by Boolean functions.
\end{proof}

Note that
there does not exist a perfect Boolean autoencoder
with $d < \lceil \log n \rceil$ because $X_n$ contains $n$ different vectors.
Note, furthermore, that the Proposition permits the use of arbitrary
Boolean functions.
In the following, we focus on what is achievable when a neural network model
with Boolean threshold functions \cite{anthony01} is used.
A function $f$: $\{0,1\}^h \rightarrow \{0,1\}$ is called
a \emph{Boolean threshold function} if it is represented as
\begin{eqnarray*}
f(\xvec) = \left\{
\begin{array}{ll}
1, & \avec \cdot \xvec \geq \theta,\\
0, & \mbox{otherwise,}
\end{array}
\right.
\end{eqnarray*}
for some $(\avec,\theta)$,
where $\avec$ is an $h$-dimensional integer vector and
$\theta$ is an integer.
We will also denote the same function as
$[\avec \cdot \xvec \geq \theta]$.
If all activation functions in a neural network are Boolean threshold
functions, the network is called a \emph{Boolean threshold network} (BTN).
In the following sections, we will consider BTNs with $L$ layers where
$L \geq 2$
(see Fig.~\ref{fig:multilayer}).
Such a BTN is represented as
$\yvec = \fvec^{(L-1)}(\fvec^{(L-2)}( \cdots \fvec^{(1)}(\xvec) \cdots ))$,
where
$\fvec^{(i)}$ is a list of activation functions for the $i+1$-th layer.
When a BTN is used as an autoencoder,
some layer is specified as the middle layer.
If the $k$th layer is specified as the middle layer,
the \emph{middle vector} $\zvec$, \emph{encoder} $\fvec$,
and \emph{decoder} $\gvec$ are defined by
\begin{eqnarray*}
\zvec & = & \fvec^{(k-1)}(\fvec^{(k-2)}(\cdots \fvec^{(1)}(\xvec) \cdots )) = \fvec(\xvec),\\
\yvec & = & \fvec^{(L-1)}(\fvec^{(L-2)}(\cdots \fvec^{(k)}(\zvec) \cdots )) = \gvec(\zvec).
\end{eqnarray*}

\begin{figure}[th]
\begin{center}
\includegraphics[width=8cm]{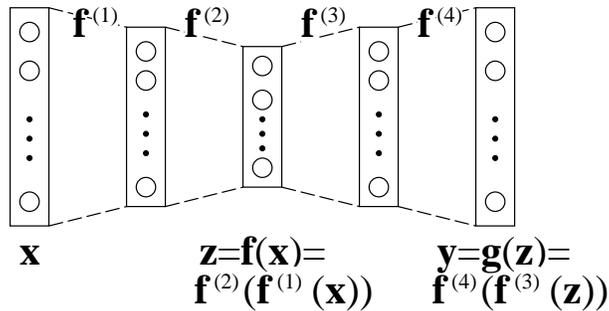}
\caption{Architecture of five layer BTN.
In this case, the third layer corresponds to the middle layer.}
\label{fig:multilayer}
\end{center}
\end{figure}

\section{How Easy is it to Encode?}

This section is devoted to the encoding setting.
First we establish, non-constructively, the existence of two encoders for
$n$ $D$-dimensional vectors:
a two-layer BTN that uses an output layer
with $O(\sqrt{D} \log n)$ nodes, and a three-layer BTN
with a hidden layer of size $D^2$ and an
output layer of size $2\lceil \log_2{n}\rceil$.
The third result reduces the number of output nodes further to
$\lceil\log n \rceil$, using a BTN of depth 4 with $O(\sqrt{n} + D)$
hidden nodes; it also lays out the architecture of this BTN.
Taken together these results suggest that encoding is relatively easy
by itself.
\smallskip
\begin{theorem}\label{t:perfecte}
	Given $X_n$ there exists a perfect encoder
	that is a two-layer BTN  with at most $\lceil 8\sqrt{2M} \ln n \rceil$ output nodes,
	where $M$ is the maximum number of 1's that any two
	vectors in $X_n$ have in common.
\end{theorem}
\begin{proof}
We use the probabilistic method to establish the existence of
a set $W$ of $d$ $D$-dimensional weight vectors $\wvec^{j}\in \{-1,1\}^{D}$
with the property that the two-layer BTN whose threshold functions
are $f_j(\xvec)=[\xvec \cdot \wvec^{j}\geq 1] ,
\ j=0,\ldots, d-1$,
yields output vectors that are all different. Denote by $\yvec^i$
the output vector
for $\xvec^i$, with $y^i_j=f_j(\xvec^i)$.

Consider a random $W$ such that each $ w^{j}_i $ has the value $-1$ or 1
with probability $\frac{1}{2}$.
Let us compute a lower bound on the probability that a given pair of
output vectors differs in a specific coordinate.
Without loss of generality (w.l.o.g.) we take the vectors to be
$\yvec^{0}$ and $\yvec^{1}$, and look at the $\ell$-th coordinate.
Because this probability does not depend on the coordinate,
since $W$ was chosen at random, let's denote it
by $p_{\{0,1\}}$.
Denote by $m_i$ the number of 1's in  $\xvec^{i}$,
and by $m_{\{0,1\}}$ the number of coordinates with value 1 for both
$\xvec^{0}$ and  $\xvec^{1}$.
W.l.o.g.
$x^0_0=x^1_0 = \cdots = x^0_{m_{\{0,1\}}-1}=x^1_{m_{\{0,1\}}-1}=1$,
$x^0_{m_{\{0,1\}}}=\cdots=x^0_{m_0-1}=x^1_{m_0}=\cdots=
x^1_{m_0+m_1-m_{\{0,1\}}-1}=1$,
all other coordinates being 0.

Consequently ${y}^0_{\ell}=[\sum_{j=0}^{m_0-1}w^{\ell}_j\geq 0]$, and
${y}^1_{\ell}=[\sum_{j=0}^{m_{\{0.1\}}-1}w^{\ell}_j+
\sum_{j=m_0}^{m_0+m_1-m_{\{0.1\}}-1}w^{\ell}_j\geq 0]$.

Suppose $m_{\{0,1\}}$ is even. Then a lower bound on $p_{\{0,1\}}$ is
$p_{\{0,1\}} \geq Prob(\sum_{j=0}^{m_{\{0,1\}}-1}w^{\ell}_j=0 )\cdot Q$, where
$Q$ stands for
\begin{align}
  &Prob(\sum_{j=m_{\{0,1\}}}^{m_0-1}w^{\ell}_j\geq 0)   \cdot
Prob(\sum_{j=m_0}^{m_0+m_1-m_{\{0,1\}}-1}w^{\ell}_j<0)+\nonumber \\
&Prob(\sum_{j=m_{\{0,1\}}}^{m_0-1}w^{\ell}_j< 0)   \cdot
Prob(\sum_{j=m_0}^{m_0+m_1-m_{\{0,1\}}-1}w^{\ell}_j\geq 0)
 .\label{eq:prob2}
\end{align}

By the soon-to-be-stated Lemma \ref{l:val},
$Prob(\sum_{j=0}^{m_{\{0,1\}}-1}w^{\ell}_j= 0) \geq \frac{1}{\sqrt{2m_{\{0,1\}}}}$.
Turning to the estimation of $Q$, Lemma \ref{l:val} implies that if one of
$m_0-m_{\{0,1\}}, m_1-m_{\{0,1\}}$ is odd then $Q=\frac{1}{2}$.
In the remaining case that both $r_0=m_0-m_{\{0,1\}}$ and $ r_1=m_1-m_{\{0,1\}}$ are even
$$Q=\frac{1}{2}-\binom{r_1}{\frac{r_0}{2}}\binom{r_1}{\frac{r_1}{2}}\frac{1}
{2^{r_0+r_1+1}}.$$
This expression assumes its smallest value, for $r_0+r_1\geq 1$,
when one of $r_0,r_1$ is 0 and the other is 2,
in which case its value is $\frac{1}{4}$.
We conclude, therefore,
that if the number of 1's that $\xvec^0,\xvec^1$ have in common is even
then $Q\geq \frac{1}{4}$ and
$p_{\{0,1\}}\geq\frac{1}{ 4\sqrt{2m_{\{0,1\}}}}
\geq \frac{1}{4\sqrt{2D}}$.

A similar computation for the case that $m_{\{0,1\}}$ is odd shows that
if number of 1's that $\xvec^0,\xvec^1$ have in common is odd then
$p_{\{0,1\}}\geq \frac{1}{4\sqrt{m_{\{0,1\}}}}$.

In summary,  the probability that a specific pair of output vectors differ
at a specific coordinate is at least
$$p=\min_{0\leq i <j \leq n-1}p_{\{i,j\}}\geq \frac{1}{4\sqrt{2M}},$$
where $M=\max_{0\leq i <j \leq n-1}m_{\{i,j\}}$,

Consequently the probability that all coordinates of two
$d$-dimensional output vectors
are identical is at most $(1-\frac{1}{4\sqrt{2M}})^d$.
To ensure that the probability that one of the $\binom{n}{2}$ pairs of vectors
$\yvec^{i_0}$ and  $\yvec^{i_1}$ is identical is less than 1
it is sufficient to choose $d$ such that
\[
\binom{n}{2}(1-\frac{1}{4\sqrt{2M}})^d < 1,
\]
i.e. we require $d\geq 8\sqrt{2M} \ln n$.
\end{proof}

\begin{remark}
In many applications the feature vectors are sparse so that $M$ is small
and the bound given in Theorem \ref{t:perfecte} can be lowered. Consider for example
the case that the number of 1s is $D^{\alpha}, \ \alpha <1$. Then it is
sufficient to set $d\geq 8\sqrt{2}D^{\frac{\alpha}{2}} \ln n$.
\end{remark}

Here is the statement of the Lemma used in the foregoing proof.
\smallskip
\begin{lemma}\label{l:val}
Let $w_j, j=0,\ldots, m-1, w_j\in \{-1,1\}$ be random variables, and let
$w=\sum_{j=0}^{m-1}w_j$.
	
\noindent If $m$ is even the following hold.
\begin{enumerate}
	\item
	$Prob(w = 0)=
	\binom{m}{m/2}\frac{1}{2^m}$.  This implies, using Stirling's approximation,
	$Prob(w = 0)\geq \frac{1}{\sqrt{2m}}$.
	\item $Prob(w \geq 0)=\frac{1}{2}(1+Prob(w = 0)).$
	\item $Prob(w \geq 2)=Prob(w <0)=\frac{1}{2}(1-Prob(w = 0)).$
\end{enumerate}
If $m$ is odd the following hold.
\begin{enumerate}
	\item $Prob(w = 1)=Prob(w = -1)=
	\binom{m}{(m-1)/2}\frac{1}{2^m}$.
This implies
	$Prob(w = 1)\geq \frac{1}{2\sqrt{m}}$.
	\item $Prob(w \geq 0)=Prob(w <0)=\frac{1}{2}$.
\end{enumerate}
\end{lemma}
\begin{proof}
It suffices to note that
if $m$ is even (odd) then $w$ is even (odd), and that $Prob(w <0)=Prob(w >0)$.
\end{proof}
\smallskip
Next we show that the number
of output nodes needed to encode the $n$ vectors can be reduced
to $2\lceil \log_2{n}\rceil$ by adding a layer of $D^2$ hidden nodes
to the BTN. To prove this we first present a result of independent interest,
that uses the power of parity functions.
\smallskip
\begin{theorem}\label{t:paritylayer}
Given $X_n$ there is a two-layer network whose activation functions
are parity functions that maps $X_n$ to $n$ different vectors
of dimension $2\lceil \log_2{n}\rceil$.
\end{theorem}

\begin{proof}
We again use the probabilistic method to establish the existence
of a network with $d$ output nodes, provided $d\geq 2 \log_2 n$.
Equip each output node $y_k$, $k=0,\ldots , d-1$
with an activation function $f_k$ constructed as follows:
assemble a set of bits, $S_k$, by adding each one
of the $D$ input variables to the set with probability $\frac{1}{2}$,
and set $f_k(\xvec)$ to be the parity of the set of values that $\xvec$
assigns to the variables in $S_k$.

Consider an arbitrary fixed output node $y_k$.
Given a pair of input vectors $(\xvec^i,\xvec^j)$,
define
\[
B^i_{i,j}  =  \{ \ell \mid \xvec^i_{\ell}=1,~\xvec^j_{\ell}=0 \},\
B^j_{i,j} =  \{ \ell \mid \xvec^i_{\ell}=0,~\xvec^j_{\ell}=1 \}.
\]
Then $f_k(\xvec^i) = f_k(\xvec^j)$ if and only if
$|B^i_{i,j} \cap S_k|$ and
$|B^j_{i,j} \cap S_k|$ are either both odd or both even.

If $|B^i_{i,j}| > 0$, then
$ Prob(|B^i_{i,j} \cap S_k|\ is\ even)= Prob(|B^i_{i,j} \cap S_k|\ is\ odd)=\frac{1}{2}$,
because $S_k$ is randomly selected.
Therefore, if $|B^i_{i,j}| > 0$ and $|B^j_{i,j}| > 0$, then
the probability that $f_k(\xvec^i) = f_k(\xvec^j)$
is ${\frac 1 2}$.
If $|B^i_{i,j}| = 0$ then $|B^j_{i,j}| > 0$, since
$\xvec^i \neq \xvec^j$,
and $Prob(f_k(\xvec^i) = f_k(\xvec^j))=Prob(|B^j_{i,j} \cap S_k|\ is\ even)=\frac{1}{2}$.

We conclude that $Prob(f_k(\xvec^i) = f_k(\xvec^j))={\frac 1 2}$,
and the probability that $Prob(\yvec^i= \yvec^j)=({\frac 1 2})^d$.
Therefore the probability that there is some pair $\xvec^i,\xvec^j$
for which $\yvec^i= \yvec^j$ is smaller than
$n^2 ({\frac 1 2})^d$, which is less than 1 if $d\geq 2 \log_2 n$.
\end{proof}
\smallskip
\begin{theorem}\label{t:threshparity}
	Given $n$ different binary vectors of dimension $D$, there is
	a three-layer BTN
	whose activation functions are threshold functions that
	maps these vectors to $n$ different vectors of dimension $2\lceil \log_2{n}\rceil$
	using $D^2$ hidden nodes.
\end{theorem}

\begin{proof}
It is well-known that a parity function on $D$ nodes can be
implemented by threshold functions
while using only one additional layer of at most $D$ nodes.
A simple way of doing this is to use a hidden layer with
$D$ nodes, $h_0,...,h_{D-1}$,
where node $h_i$ has value 1 if the input
contains at least $i+1$ 1's, i.e.
$h_i(\xvec)=[\evec \cdot \xvec \geq i+1]$, where $\evec=(1,...,1)$.
The output node that computes the parity of the input
has the activation function
$[\sum_{i=0}^{D-1} (-1)^{i-1} h_i\geq 1]$.
The Theorem follows therefore from Theorem \ref{t:paritylayer} when
each of the parity functions it uses is implemented this way.
\end{proof}

We turn now to results that also provide constructions. In one form or another all
use the idea of having the encoder map a vector $\textbf{x}^i$ to its index,
as done in Proposition \ref{p:map}.
That Proposition permits the use of
any Boolean mapping, so that the meaning of the index of the vector is of no importance.
Here, however, only mappings arising from BTNs are employed, and the meaning
of the index,
as provided by the following simple observation, is going to be of central importance.
\smallskip
\begin{lemma}cf. \cite{zhang17}.
	Given
	$X_n$ there is a vector $\textbf{a}$ such that
	$\textbf{a}\cdot \textbf{x}^i \neq \textbf{a}\cdot \textbf{x}^j$ if $i\neq j$.
	Furthermore, this vector can be assumed to have integer coordinates.
\end{lemma}
The reason  is that the vectors in $X_n$ are all different. We fix $\avec$,
and assume henceforth that $X_n$ is sorted,
\begin{equation}\label{e1}
b^0 < \cdots <b^{n-1},\ b_i=\textbf{a}\cdot \textbf{x}^i.
\end{equation}

When a vector $\textbf{x}^i$ is mapped to its index, $i$ will be represented either in binary,
or by the $s$-dimensional step-vector $\textbf{h}^i[s]$, for some $s$, defined by
\begin{equation}\label{e2}
\textbf{h}^i[s]_j=1 \mbox{ if } 0\leq j\leq i,\ 0 \mbox{ if } i<j<s.
\end{equation}
\smallskip
Our first result in this direction exemplifies encoding with step-vectors.
\smallskip
\begin{theorem}\label{t:2-layer encoder}
	Given $X_n$ there is a two-layer BTN
	that maps $\textbf{x}^i$ to $\textbf{h}^i[n],\ i=0,\ldots, n-1$.
\end{theorem}
\begin{proof}
	Let the vector of threshold functions for the $n$ nodes in the output  layer be
	$\textit{G}_j(\textbf{x})=[\textbf{a}\cdot \textbf{x}\geq b_j]$.
	Clearly
	$\textit{G}_j(\textbf{x}^i)=[\textbf{a}\cdot \textbf{x}^i\geq b_j]$
	has value 1 if and only if $j\leq i$, i.e.
	$\textbf{\textit{G}}(\textbf{x}^i)=\textbf{h}^i$.
\end{proof}

\smallskip
We show next that the number of output nodes of the encoder can
be reduced to $2\lceil \sqrt{n} \rceil$ if a hidden layer
with only $ \lceil \sqrt{n} \rceil +D$ nodes is added. Henceforward we
denote $r=\lceil \sqrt{n} \rceil$.
\smallskip

\begin{theorem}\label{t:3-layer sqrt encoder}
	Given $X_n$ satisfying	equation (\ref{e1}),
	there is a three-layer network that maps $\textbf{x}^i$
    to $(\boldsymbol{h}^k[r], \boldsymbol{h}^{\ell}[r])$
    where $k$ and $\ell$ are such that $i=kr+\ell$, with $0\leq \ell<r$.
	The network has $D$ input nodes $x_j,\ j=0,\ldots, D-1$,
	$r+D$ hidden nodes
	${\alpha}^1_i,\ i=0,\ldots, r-1$ and ${\alpha}^2_j,\ j=0,\ldots, D-1$,
	and $2r$ output nodes
	$\beta^1_i, \beta^2_i,\ i=0,\ldots, r-1$.
\end{theorem}

\begin{proof}
	For ease of exposition we assume that
	$\log n={2m}$, and $r=\sqrt{n}=2^m$.
	
	The nodes $\alpha^2_j$ simply copy the input, ${\alpha}^2_j=x_j,\ j=0,\ldots, D-1$.
	To define the activation functions of $\alpha^1_i$
	divide the  interval $[b_0,b_{n-1}]$ into $r $
	consecutive subintervals $[s_0,s_1-1],\ [s_1,s_2-1],\ldots , [s_{r-1}, b_{n-1}]$
	each containing $ r$ values $b_i$, i.e. $s_0=b_0, s_1=b_{r}, \ldots , s_{r-1}=b_{(r-1)r}$, and $s_{r}=b_n+1$.
	The
	activation function of $\alpha^1_i$ is $[\avec \cdot \xvec \geq s_i]$,
	$0\leq i \leq r-1$. Clearly, if $h=kr+\ell$ then
	$\alpha^1_i=1$ if and only if $i\leq k$, i.e. $\boldsymbol{\alpha}^1=\boldsymbol{h}^k[r]$.
	
	The output node $\beta^1_i$ copies $\alpha^1_i,\ 0\leq i \leq r-1$.
	The remaining output nodes, $\boldsymbol{\beta}^2$,  represent the ordinal number of $\avec \cdot \xvec$ within the $k$-th subinterval. To implement this part of the mapping we
	follow \cite{srk91} in employing the ingenious technique of telescopic sums, introduced by \cite{minnick61},
	and define
	\begin{align*}
	t_i=b_j \alpha_0 &+ (b_{r+i}-b_i)\alpha_1+\ldots\\  \ldots &+(b_{(r-1)r+i}-b_{(r-2)r+i})\alpha_{r-1},i=0,\ldots,r-1.
	\end{align*}
	The important thing to note here is that if $\avec \cdot \xvec$ falls in the
	$k$-th subinterval, so that $s_{k}\leq \avec \cdot \xvec<s_{k+1}$, then
	$t_i$ has the value $b_{kr+i}$, because
	$\alpha_0=\ldots =\alpha_k=1$, while $\alpha_{i}=0, \ i>k$.
	
	The activation function of $\beta^2_i$ is
	$[\avec \cdot \xvec -t_i\geq 0]$, $i=0,\ldots , r-1$.
	Observe that if  $\avec \cdot \xvec=b_{kr+\ell}$, then
	$\beta^2_i=[b_{kr+\ell}-b_{kr+i}\geq 0]$, i.e. $\beta^2_i=1$ if $i\leq \ell$,
	and $\beta^2_i=0$ if $i> \ell$.
\end{proof}
\smallskip
It is relatively easy to convert the BTNs constructed in Theorems \ref{t:2-layer encoder}
and \ref{t:3-layer sqrt encoder} into BTNs that map $X_n$ to a set of vectors with the minimum possible dimension, $\lceil \log n \rceil$, by adding a layer of that size. The technique for doing
so is encapsulated in the following Lemma. Its formulation is slightly more
general than strictly needed here in order to serve another purpose later on.
\smallskip
\begin{lemma}\label{l:gmap}
	Given a set of $s\leq 2^d$ different $d$-dimensional binary vectors
	$Z=\{\zvec^i,\ i=0,\ldots,s-1 \}$ there is a two-layer BTN
	that maps $\textbf{h}^i[s]$ to $\zvec^i$.
\end{lemma}
\begin{proof}
To construct the vector of threshold functions for the output nodes,
$\textbf{\textit{H}}(\textbf{h}[s])$, set $w^j_0=z^0_j, j=0,\ldots, d-1$,
$$w^j_i=z^{i}_j-z^{i-1}_j,i=1,\ldots, s-1,\ j=0,\ldots, d-1,$$
and $\textit{H}_j(\textbf{h}[s])=[\textbf{w}^j\cdot \textbf{h}[s]\geq 1].$
It is easily verified that
$\textit{H}_j(\textbf{h}[s]^i)=[z^i_j\geq 1]=z^i_j$, i.e.
$\textbf{\textit{H}}(\textbf{h}[s]^i)=\textbf{z}^i$ as desired.	
\end{proof}
\smallskip
Consider the layer constructed in the proof of the Lemma when
$\zvec^i$ is the $\lceil \log n \rceil$-dimensional binary representation of $i$ and $s=n$.
By adding this layer to the BTN constructed in Theorem \ref{t:2-layer encoder} we get the following result.
\smallskip
\begin{theorem}\label{t:3-layer encoder}
	Given $X_n$ there is a three-layer BTN with $n$ nodes in the second layer
	that maps $\xvec^i$ to the $\lceil \log n \rceil$-dimensional binary representation
	of $i$, $i=0, \ldots, n-1$.
\end{theorem}
\smallskip
Similarly, we can add to the BTN constructed in Theorem \ref{t:3-layer sqrt encoder}
a layer that consists of a first part and a second part, both constructed according to the Lemma. The first part connects the nodes $\boldsymbol{\beta}^1$ to a
$\lceil \log r \rceil$-dimensional binary counter $\boldsymbol{\gamma}^1$, and the second part connects the nodes $\boldsymbol{\beta}^2$ to another $\lceil \log r \rceil$-dimensional binary counter $\boldsymbol{\gamma}^2$. Thus, on input $\textbf{x}^i$
the nodes $\boldsymbol{\beta}^1$ contain $\textbf{h}^k[r]$ which is mapped
to the $\lceil \log n \rceil$-dimensional binary representation of $k$ in
$\boldsymbol{\gamma}^1$, and the nodes $\boldsymbol{\beta}^2$ contain $\textbf{h}^{\ell}[r]$ which is mapped
to the $\lceil \log n \rceil$-dimensional binary representation of $\ell$ in
$\boldsymbol{\gamma}^2$. Recalling that $i=kr+\ell$, with $0\leq \ell<r$,
it follows that, in this case, $\boldsymbol{\gamma}^1$ and
$\boldsymbol{\gamma}^2$ taken together contain the $\lceil \log{n} \rceil$-dimensional
binary representation of $i$.
This construction establishes the following Theorem.
\smallskip
\begin{theorem}\label{t:4-layer sqrt encoder}
	Given $X_n$ there is a four-layer BTN
	with $3\lceil \sqrt{n} \rceil +D$ hidden nodes that maps $\xvec^i$ to
	the $\lceil \log n \rceil$-dimensional binary representation
	of $i$, $i=0, \ldots, n-1$.
\end{theorem}

\section{Difficulty of Decoding}

\begin{theorem}\label{t:diff}
	For each $d$ there exists a set of $n=2^d$ different $N$-dimensional bit-vectors
	$Y=\{\textbf{y}^i\}$ with the following properties:
	\begin{enumerate}
		\item There exists a 2-layer BTN that maps $Y$ to $\{0,1\}^d$.
		\item There does not exist a 2-layer BTN that maps $\{0,1\}^d$ to $Y$.
	\end{enumerate}
\end{theorem}
\begin{proof}
Set $N= \binom{n}{2}$ and construct $Y$ as follows:
\begin{enumerate}
	\item index the coordinates of the vectors by $(i,j),\ 0\leq i<j\leq n-1$;
	\item set ${y}^k_{(i,j)}=1$ if and only if $i$ or $j$ is $k$.
\end{enumerate}

A BTN that encodes these vectors with $d$ bits is constructed as follows.
The most significant bit is generated by the threshold function
$$\sum_{i=\frac{n}{4}}^{\frac{n}{2}-1}y_{(2i,2i+1)}\geq 1.$$
The next most significant bit is generated by
$$\sum_{i=\frac{3n}{8}}^{\frac{n}{2}-1}y_{(2i,2i+1)}+
\sum_{i=\frac{n}{8}}^{\frac{n}{4}-1}y_{(2i,2i+1)} \geq 1.$$
The successive digits are generated in similar fashion, each time using a sum
of  $\frac{n}{4} $ input bits, until at last the least significant is generated by
the slightly different threshold function
$$y_{(1,3)}+y_{(5,7)}+\cdots y_{(n-3,n-1)}\geq 1.$$
It is not difficult to verify that this 2-layer BTN maps $\textbf{y}^i$ to the
$d$-bit representation of $i$.

To illustrate, consider the case $d=3,\ n=8,\ N=28$.
These vectors are encoded by the threshold functions $d_0$ (most significant digit) $d_1$ and $d_2$ (least significant digit)
\begin{align*}
d_0(\yvec)&=[y_{(4,5) }+ y_{(6,7)}\geq 1],\\
d_1(\yvec)&=[y_{(2,3) }+ y_{(6,7)}\geq 1],\\
 d_2(\xvec)&=[y_{(1,3) }+ u_{(5,7)}\geq 1].
\end{align*}
For example, since in $\yvec^7$ the coordinates $(0,7),\ldots, (6,7)$ are 1 it has $d_0(\yvec^7)=d_1(\yvec^7)=d_2(\yvec^7)=1$, i.e. its representation is $(1,1,1)$.
Similarly, since
$y^1_{(6,7)}=y^1_{(4,5)}=y^1_{(2,3)}=0$
whereas $y^1_{(1,3)}=1$,  $d_0(\yvec^1)=d_1(\yvec^1)=0, d_2(\yvec^1)=1$, so that
$\yvec^1$ has the representation $(0,0,1)$.

To prove the second assertion, suppose to the contrary that
there does exist a BTN mapping $\{0,1\}^d$ to $Y$, and
that it maps
$(0,\ldots, 0)$ to $\textbf{y}^k$ and  $(1, \ldots,1)$ to $\textbf{y}^{\ell}$.
Consider the threshold function for the bit $x_{(k,\ell)}$ of the output layer.
Since by construction $y^k_{(k,\ell)}=y^{\ell}_{(k,\ell)}=1$ while $y^i_{(k,\ell)}=0$ for all
$i\neq k,\ \ell$, it follows that the threshold function for this bit separates the two $d$-dimensional vectors $(0,\ldots, 0)$ and $(1, \ldots,1)$ from the remainder of  $\{0,1\}^d$,
a contradiction to the well-known fact that such a threshold function does not exist.
\end{proof}
\bigskip
\begin{corollary}
	In general there may not exist an auto-encoding three-layer BTN if the
	middle layer has only $\lceil \log n \rceil$ nodes.
\end{corollary}
\section{Perfect Encoding and Decoding}
In this section we construct four autoencoders by adding decoders to the encoders constructed
in Theorems \ref{t:2-layer encoder}, \ref{t:3-layer sqrt encoder}, \ref{t:3-layer encoder},
and \ref{t:4-layer sqrt encoder}.
The autoencoder based on
Theorem \ref{t:2-layer encoder} is described in the next Theorem.
\smallskip
\begin{theorem}\label{t:3-layer auto}
	Given $X_n$ satisfying equation (\ref{e1}),
	there is a three-layer perfect BTN autoencoder with $n$ nodes in the middle layer.
\end{theorem}
\begin{proof}
To the two-layer encoder of Theorem \ref{t:2-layer encoder} add the decoder
that maps  $\textbf{h}^i[n]$ to $\textbf{x}^i,\ i=0,\ldots, n-1$. That decoder can be
obtained from the construction of Lemma	\ref{l:gmap} by setting $\zvec^i=\textbf{x}^i$
and $s=n$.
\end{proof}
\smallskip
To obtain an autoencoder based on Theorem \ref{t:3-layer sqrt encoder} will take more doing.
\smallskip
\begin{theorem}\label{t:5-layer sqrt auto}
		Let $r=\lceil \sqrt{n} \rceil$. Given $X$ satisfying equation (\ref{e1}),
	there is a five-layer BTN perfect autoencoder with $r+D$  nodes  in the second layer, $2r$ nodes in the middle hidden layer, and $rD$ nodes in the
	fourth layer.
\end{theorem}
\begin{proof}
On top of the encoder constructed in the proof of Theorem \ref{t:3-layer sqrt encoder}
we add
a decoder that is a 3-layer network with the $2r$ input nodes
$\beta^1_i, \beta^2_i,\ i=0,\ldots, r-1$,
$rD$ hidden nodes $\eta_{i,j},\ i=0,\ldots,r-1,\ j=0,\ldots, D-1$,
and $D$ output nodes $y_j,\ j=0,\ldots, D-1$,
see Fig.~\ref{fig:5-layer auto}.
To simplify the description we add a dummy node $\beta^2_r=0$.

\begin{figure}[th]
	\begin{center}
		\includegraphics[width=12cm]{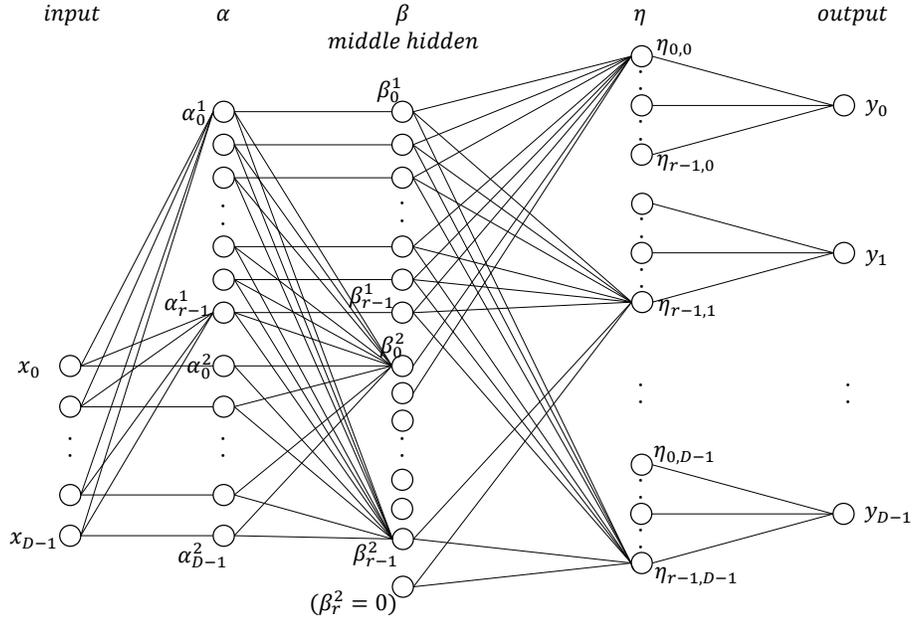}
		\caption{Perfect Encoder/Decoder BTN.}
		\label{fig:5-layer auto}
	\end{center}
\end{figure}

We will see that this decoder
outputs $\textbf{y}=\textbf{x}^{kr+\ell}$
when given input $(\boldsymbol{h}^k[r], \boldsymbol{h}^{\ell}[r])$,
as desired.
Equip the node $\eta_{i,j}$ with the activation function
\begin{align}\label{egamma}
[x^i_j\beta^1_0&+(x^{r+i}_j-x^i_j)\beta^1_1+\ldots \nonumber\\ \ldots &+(x^{r(r-1)+i}_j-x^{r(r-2)+i}_j)\beta^1_{r-1}+\beta^2_i-\beta^2_{i+1}\geq 2].
\end{align}
Note that if $\boldsymbol{\beta}^1=\boldsymbol{h}^k[r]$ then
\begin{align*}
x^i_j\beta^1_0&+(x^{r+i}_j-x^i_j)\beta^1_1+\ldots\\ \ldots &+(x^{r(r-1)+i}_j-x^{r(r-2)+i}_j)\beta^1_{r-1}=x^{kr+i}_j.
\end{align*}
Note further that if $\boldsymbol{\beta}^2=\boldsymbol{h}^{\ell}[r]$
 then $\beta^2_i-\beta^2_{i+1}=0$ unless $ i=\ell$.
Hence, when $(\boldsymbol{\beta}^1,\boldsymbol{\beta}^2)=
(\boldsymbol{h}^k[r],\boldsymbol{h}^{\ell}[r])$, the value
of $\eta_{i,j}$ is $x^{kr+\ell}_j$ if $i=\ell$ and 0 otherwise,
according to the activation function (\ref{egamma}).

To complete the definition of the network we equip node $y_j$ with the activation function
\begin{equation*}\label{ey}
[\sum^{r-1}_{i=0} \eta_{i,j}\geq 1].
\end{equation*}
It follows that when the network is given the input $\boldsymbol{\beta}^1$, $\boldsymbol{\beta}^2$
as specified in the
Theorem,
$y_j$ has value 1 if and only if
$x^{kr+\ell}_j=1$, i.e. $y_j=x^{kr+\ell}_j$.
\end{proof}
\smallskip

We will now show how to replace the middle layer of each of the autoencoders of
Theorems \ref{t:3-layer auto} and \ref{t:5-layer sqrt auto} with a three-layer
autoencoder whose middle layer has
the minimum possible dimension, $\lceil \log n \rceil$.
Observe that the middle layer of the autoencoder constructed in the proof of
Theorem \ref{t:3-layer auto} has $n$ nodes and that the only values that this
layer assumes are  $\textbf{h}^i[n], i=0,\ldots, n-1$.
The middle layer of the autoencoder constructed in the proof of
Theorem \ref{t:5-layer sqrt auto} has two sets of  $r=\lceil \sqrt{n} \rceil$ nodes and the only values that each of these sets can assume are  $\textbf{h}^i[r], i=0,\ldots, r-1$.
In each case, therefore, the appropriate three-layer autoencoder to replace the
middle layer can be constructed from one or two autoencoders on $s$ nodes which
only assume the values $\textbf{h}^i[s], i=0,\ldots, s-1$.
\smallskip
\begin{lemma}\label{l:autoencoder}
	There is a three-layer BTN that autoencodes the set of vectors $\textbf{h}^i[s], i=0,\ldots, s-1$ and has a middle layer of size $\lceil \log{s} \rceil$.
\end{lemma}
\begin{proof}
	For ease of exposition we assume that
	$ s={2^m}$.
We construct a BTN with $s$ input nodes $\boldsymbol{\beta}$,
$m$ hidden nodes $\boldsymbol{\gamma}$, and
$s$ output nodes $\boldsymbol{\delta}$.

On input $\textbf{h}^i[s]$ the middle layer
$\boldsymbol{\gamma}$ contains the binary representation
of $i$.
The decoding layer is therefore straightforward:
the activation function for $\delta_j, j=0,\ldots, s-1$ is
$[\sum_{h=0}^{m-1} \gamma_h 2^h\geq j]$.

The activation functions for $\gamma_j, j=0,\ldots ,m-1$
can be computed by the method of Lemma \ref{l:gmap}:
\begin{align*}
\gamma_{m-1}:&\ [\beta_{\frac{s}{2}}\geq 1];\\
\gamma_{m-2}:&\ [\beta_{\frac{s}{4}}- \beta_{\frac{2s}{4}}+\beta_{\frac{3s}{4}}\geq 1];\\
\gamma_{m-3}:&\ [\beta_{\frac{s}{8}}-\beta_{\frac{2s}{8}}+
\beta_{\frac{3s}{8}}-\beta_{\frac{4s}{8}-1}+\ldots +
\beta_{\frac{7s}{8}}\geq 1];\\
& .\\
& . \\
\gamma_{0}:& [\beta_1-\beta_2+\beta_3-\beta_4+\ldots + \beta_{s-1}\geq 1].
\end{align*}
Note that $\beta_0=\delta_0=1$ for all inputs $\textbf{h}^i[s]$, since by definition $\textbf{h}^i[s]_0=1$, and the activation function of $\delta_0$ is $[\sum_{h=0}^{m-1} \gamma_h 2^h\geq 0]$.
\end{proof}
\smallskip

Replacing the middle layers of the autoencoders constructed in
Theorems \ref{t:3-layer auto} and  \ref{t:5-layer sqrt auto} with the
appropriate three-layer BTN constructed according to the Lemma
yields the following two results.
\smallskip
\begin{theorem}\label{t:5-layer auto}
	Given $X_n$ satisfying equation (\ref{e1}),
	there is a five-layer perfect BTN autoencoder with $n$ nodes
	in the second and fourth layer and
	$\lceil\log n \rceil$ nodes in the middle layer.
\end{theorem}
\smallskip
\begin{theorem}\label{t7-layer}
	Given
	$X_n$ satisfying	equation (\ref{e1}),
	there is a seven-layer perfect BTN autoencoder with
	$2\lceil\log \sqrt{n} \rceil$
	nodes in the middle hidden layer,
	and $(D+5)\lceil \sqrt{n}\rceil +D$
	nodes in the other hidden layers.
\end{theorem}

\section{Conclusion}

In this paper, we have studied the compressive power of autoencoders
mainly
within the Boolean threshold network model by exploring the
existence of encoders and of autoencoders that
map distinct input vectors to distinct vectors in lower-dimensions.
It should be noted that our results are not necessarily optimal except for
Proposition \ref{prop:bool}.
The establishment of lower bounds and the reduction of the number of layers or nodes
are left as open problems.

Although we focused on the existence of injection mappings,
conservation of the distance is another important factor
in dimensionality reduction.
Therefore, it should be interesting and useful to study
autoencoders that approximately conserve the distances
(e.g., Hamming distance) between input vectors.

Another important direction is to use more practical models of
neural networks such as those with
sigmoid functions and Rectified Linear Unit (ReLU) functions.
In such cases, real vectors need to be handled and thus
conservation of the distances between vectors should also be analyzed.

\end{document}